\documentclass{article}

\pdfoutput=1

\makeatletter
\renewcommand*{\@fnsymbol}[1]{\ifcase#1\or*\else\@arabic{\numexpr#1-1\relax}\fi}
\makeatother

\usepackage[preprint]{neurips_2019}

\usepackage[utf8]{inputenc} 
\usepackage[T1]{fontenc}    
\usepackage{url}            
\usepackage{booktabs}       
\usepackage{amsfonts}       
\usepackage{nicefrac}       
\usepackage{microtype}      
\usepackage{color}
\usepackage{xcolor}

\newcommand{\ifcomments}{\iftrue}
\ifx\condtion\undefined
\fi

\usepackage{epsfig, endnotes, amsmath, amsfonts, amssymb}
\usepackage{subfigure}
\usepackage{amsmath, amsfonts, amssymb, amsthm}
\usepackage{txfonts} 
\usepackage{verbatim}
\usepackage{url}
\usepackage{color}
\usepackage{algorithm2e}
\usepackage[noend]{algpseudocode}
\usepackage{multirow}

\usepackage{siunitx}
\ifx\BlackBox\undefined
\newcommand{\BlackBox}{\rule{1.5ex}{1.5ex}}  
\fi
\ifx\proof\undefined
\newenvironment{proof}{\par\noindent{\bf Proof\ }}{\hfill\BlackBox\\[2mm]}
\fi

\newcommand\shortsection[1]{\vspace{6pt}{\noindent\bf #1.}}

\newtheorem{theorem}{Theorem}[section]

\theoremstyle{definition}
\newtheorem{definition}{Definition}[section]

\makeatletter
\def\url@leostyle{%
  \@ifundefined{selectfont}{\def\UrlFont{\sf}}{\def\UrlFont{\small\sffamily}}}
\makeatother
\makeatletter
\def\url@beostyle{%
  \@ifundefined{selectfont}{\def\UrlFont{\sf}}{\def\UrlFont{\scriptsize\sffamily}}}
\makeatother

\usepackage[toc,page]{appendix}

\usepackage{booktabs}
\usepackage{siunitx}

\begin{document}
\thispagestyle{empty}
\title{Certifying Joint Adversarial Robustness for~Model~Ensembles}
\author{Mainuddin Ahmad Jonas \\ 
University of Virginia\\
{\sf\small maj2bh@virginia.edu} 
       \And David Evans\\
       University of Virginia\\
   {\sf\small evans@virginia.edu}}
\date{}

\maketitle

\begin{abstract}

Deep Neural Networks (DNNs) are often vulnerable to adversarial examples. Several proposed defenses deploy an ensemble of models with the hope that, although the individual models may be vulnerable, an adversary will not be able to find an adversarial example that succeeds against the ensemble. Depending on how the ensemble is used, an attacker may need to find a single adversarial example that succeeds against all, or a majority, of the models in the ensemble.  The effectiveness of ensemble defenses against strong adversaries depends on the vulnerability spaces of models in the ensemble being disjoint. We consider the joint vulnerability of an ensemble of models, and propose a novel technique for certifying the joint robustness of ensembles, building upon prior works on single-model robustness certification. We evaluate the robustness of various models ensembles, including models trained using cost-sensitive robustness to be diverse, to improve understanding of the potential effectiveness of ensemble models as a defense against adversarial examples. 
\end{abstract}

\section{Introduction}

Deep Neural Networks (DNNs) have been found to be very successful at many tasks, including mage classification~\cite{alexnet,he2016deep}, but have also been found to be quite vulnerable to misclassifications from small adversarial perturbations to inputs~\cite{szegedy2014intriguing,goodfellow2015explaining}. Many defenses have been proposed to protect models from these attacks.  Most focus on making a single model robust, but there may be fundamental limits to the robustness that can be achieved by a single model~\cite{schmidt2018, gilmer2018adversarial,fawzi2018adversarial,mahloujifar2018curse,shafahi2018adversarial}. Several of the most promising defenses employ multiple models in various ways~\cite{feinman2017,tramer2017,meng2017magnet,xu2017feature,pang2019}. These ensemble-based defenses work on the general principle that it should be more difficult for an attacker to find adversarial examples that succeed against two or more models at the same time, compared to attacking a single model.  However, an
attack crafted against one model may be successful against a different model trained to perform the same task. This leads to a notion of \emph{joint vulnerability} to capture the risk of adversarial examples that compromise a set of models, illustrated in Figure~\ref{fig:joint}.  Joint vulnerability makes ensemble-based defenses less
effective. Thus, reducing joint vulnerability of models is important to ensure stronger ensemble-based defenses. 

\begin{figure}[tb]
\centering
\includegraphics[width=0.4\textwidth]{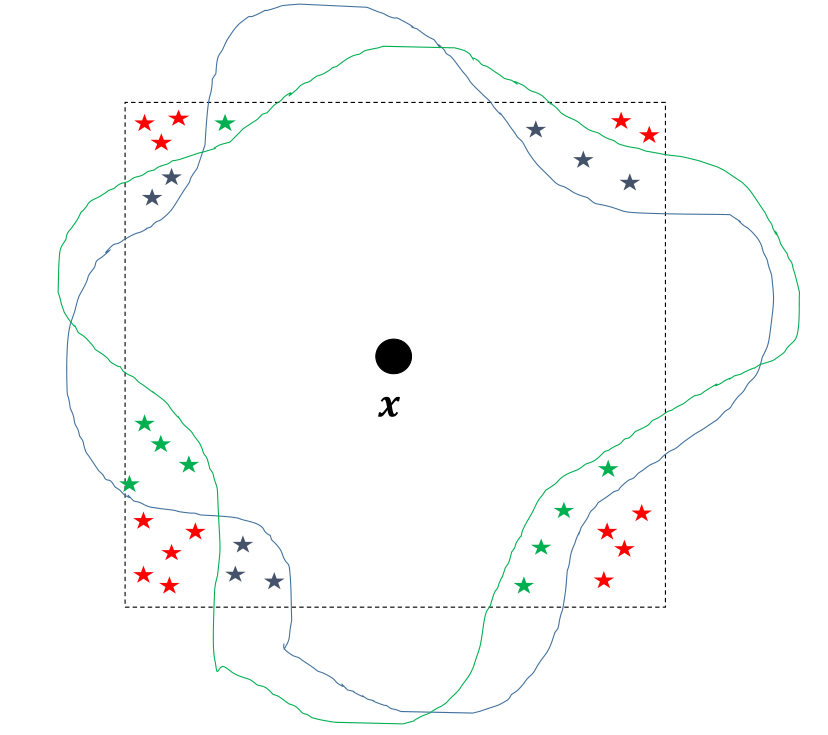}
\caption{Illustration of joint adversarial vulnerability of two binary classification models. The seed input is $x$, and the dotted square box around represents its true decision boundary. The blue and green lines describe the decision boundaries of two models. The models are jointly vulnerable in the regions marked with red stars, where both models consistently output the (same) incorrect class.}
\label{fig:joint}
\end{figure}

Although the above ensemble defenses have shown promise when evaluated against experimental attacks, these attacks often assume adversaries do not adapt to the ensemble defense, and no previous work has certified the joint robustness of an ensemble defense. On the other hand, several recent works have developed methods to certify robustness for single models~\cite{wong2018provable,tjeng2017,gowal2018}. In this work, we introduce methods for providing robustness guarantees for an ensemble of models buidling upon the approaches of Wong and Kolter~\cite{wong2018provable} and Tjeng et al.~\cite{tjeng2017}. 

\shortsection{Contributions}
Our main contribution is a framework to certify robustness for an ensemble of models against adversarial examples. We define three simple ensemble frameworks (Section~\ref{sec:ensembletypes})
and provide robustness guarantees for each of them, while evaluating the tradeoffs between them. We propose a novel technique to extend prior work on single model robustness to verify joint robustness of ensembles of two or more models (Section~\ref{sec:certifying}). Second, we demonstrate that the cost-sensitive training  approach~\cite{zhang2018} can be used to train diverse robust models that can be used to certify a high fraction of test examples (Section~\ref{sec:experiments}). Our results show that, for the MNIST dataset, we can train diverse ensembles of two, five and ten models using different cost-sensitive robust matrices. When these diverse models are combined using our ensemble frameworks, the ensembles can be used to certify a larger number of test seeds compared to using a single overall-robust model. For example, 78.1\% of test examples can be certified robust for two-model averaging ensemble and 85.6\% for a ten-model ensemble, compares with 72.7\% for a single model. We further show that use of ensemble models do not significantly reduce the model's accuracy on benign inputs, and when rejection is used as an option, can reduce the error rate to essentially zero with a 9.7\% rejection rate.

\section{Background and Related Work}
In this section, we briefly introduce adversarial examples, provide background on robust training and certification, and describe defenses using model ensembles.

\subsection{Adversarial Examples} 

Several definitions of adversarial example have been proposed. For this paper,
we use this definition~\cite{biggio2013defn, goodfellow2015explaining}: given a model $M$,
an input $x$, a distance metric $\Delta$, and a distance measure $\epsilon$, an
\emph{adversarial example} for the input $x$ is $x'$ where $M(x') \neq M(x)$ and
$\Delta(x, x') \leq \epsilon$.

In recent years, there has been a significant amount of research on adversarial
examples against DNN models, including attacks such as
FGSM~\cite{goodfellow2015explaining}, DeepFool~\cite{moosavi2016deepfool},
PGD~\cite{madry2017towards}, Carlini-Wagner~\cite{carlini2017}, and
JSMA~\cite{papernot2016limitations}. The FGSM attack works by taking the signs of the
gradient of the loss with respect to the input $x$ and adding a small
perturbation to the direction of loss for all input features. This simple strategy is surprisingly successful. The PGD attack is considered a very
strong state-of-the-art attack. It is essentially an iterative version of FGSM,
where instead of just taking one step many smaller steps are taken subject to some constraints and with some randomization. 

One interesting property of these attacks is that the adversarial examples they find are often transferable~\cite{evtimov2017robust} --- a successful attack against one model is often successful against a second model. Transfer attacks enable black-box attacks where the adversary does not have full access to the target model. More importantly for our purposes, they also demonstrate that an adversarial example found against one model is also effective against other models, so can be effective against ensemble-based defenses~\cite{xie2017,tramer2017}. In our work, we consider the threat model where the adversary has white-box access to all of the models in the ensemble and knowledge of the ensemble construction.

\subsection{Robust training} 

While many proposed adversarial examples defenses look promising, adaptive attacks that compromise defenses are nearly always found~\cite{tramer2020adaptive}.  The failures of ad hoc defenses motivate increased focus on robust training and provable defenses.  Madry et
al.~\cite{madry2017towards}, Wong et al.~\cite{wong2018provable}, and Raghunathan et
al.~\cite{raghunathan2018certified} have proposed robust training methods to defend
against adversarial examples. Madry et al.\ use the PGD attack to find
adversarial examples with high loss value around training points, and then
iteratively adversarially train their models on those seeds. Wong et al.\ define
an adversarial polytope for a given input, and robustly train the model to
guarantee adversarial robustness for the polytope by reducing the problem into a
linear programming problem. These works focus on single models; we propose a way to make \emph{ensemble} models jointly robust through training a set of models to be both robust and diverse.

\subsection{Certified robustness} 
\label{sec:cert}
Several recent works aim to provide guarantees of robustness for  models
against constrained adversarial
examples~\cite{wong2018provable,tjeng2017,raghunathan2018certified,cohen2019certified}. All of
these works provide certification for individual models. A model $M(\cdot)$ is
certifiably robust for an input $x$, if for all $x'$ where $\Delta(x, x') \leq
\epsilon$, $M(x')$ is robust.  We 
extend these techniques for ensemble models. In particular, we extend Tjeng
et al.'s~\cite{tjeng2017} MIP verification technique and Wong et
al.'s~\cite{wong2018provable} convex adversarial polytope method. Both techniques are based on using linear programming to calculate a bound on outputs given the allowable
input perturbations, and using those output bounds to provide robustness
guarantees. MIPVerify uses mixed integer linear programming solvers, which are computationally very expensive for deep neural networks. To get around this issue, Wong et
al.~\cite{wong2018provable} use a dual network formulation of the original network that 
over-approximates the adversarial region, and apply widely used techniques such as stochastic gradient descent to the solve the
optimization problem efficiently. This can scale to larger networks and provides a sound certificate, but may fail to certify robust examples because of the over-approximation.



\subsection{Ensemble models as defense} 

In classical machine learning, there has been extensive work on ensemble of
models and also diversity measures.  Kuncheva~\cite{kuncheva2003} provides a
comparison of those measures and their usefulness in terms of ensemble accuracy.
However, both the diversity measures and the evaluation of their usefulness was
done in the benign setting. The assumptions that are valid in the benign
setting, such as, the independent and identically distributed inputs no longer
applies in the adversarial setting.  

In the adversarial setting, there have been
several proposed ensemble-based defenses~\cite{feinman2017,tramer2017,pang2019}
that work on the principle of making models diverse from each other.  Feinman et
al.~\cite{feinman2017} use randomness in the dropout layers to build an ensemble
that is robust to adversarial examples. Tramer et al.~\cite{tramer2017} use
ensembles to introduce diversity in the adversarial examples on which to train a
model to be robust.  Pang et al.~\cite{pang2019} promote diversity among
non-maximal class prediction probabilities to make the ensembles diverse. Sharif et al.~\cite{sharif2019n} proposed the \emph{n}--\emph{ML} appproach for adversarial defense. They explicitly train the models in
the ensemble to be diverse from each other, and show experimentally that it leads to robust ensembles. Similarly, Meng et
al.~\cite{meng2020} have shown than an ensemble of $n$ weak but diverse models
can be used as strong adversarial defense. While all of the above works focus on
making models diverse, they evaluate their ensembles using existing attack methods.
None of these prior works have attempted to provide any certification of their
diverse ensemble models against adversaries.

\section{Ensemble Defenses}\label{sec:ensembletypes}
Our goal is to provide robustness guarantees against adversarial examples
for an ensemble of models. The effectiveness of a ensemble defense depends on
the models used in the ensemble and how they are combined. 

First, we present a general framework for ensemble defenses.  Next, we define three different ensemble composition frameworks: \emph{unanimity}, \emph{majority}, and
\emph{averaging}. Section~\ref{sec:certifying} describes the techniques we use to certify each type of ensemble framework.  In  Sections~\ref{subsec:mip} and~\ref{subsec:convex}, we talk about different ways to train the individual models in these ensemble frameworks, and discuss the results. Our methods do not make any assumptions about the models in an ensemble, for example, that they are pre-processing the input in some way and then running the same model. This means our frameworks are general purpose and agnostic of the input domain, but we cannot handle ensemble mechansisms that are nondeterministic (such as sampling Gaussian noise around the input~\cite{salman2020blackbox}, which can only provide probabilistic guarantees).

\shortsection{General frame of ensemble defense}
We use $\mathcal{M}(x)$ to represent the output of a model ensemble, composed of $n$ models, $M_1(\cdot),  M_2(\cdot), \ldots, M_n(\cdot)$ that are composed using one of the composition mechanisms. Furthermore, given an input $x$, true output class $t$, and the output of the ensemble $\mathcal{M}(x)$, we use a \emph{decision function} $\mathcal{D}(\mathcal{M}(x), t)$ to decide whether the given input $x$ is adversarial, benign, or rejected. Functions $\mathcal{M}$ and $\mathcal{D}$ together define an ensemble defense framework. We discuss three such frameworks in this paper.

\shortsection{Unanimity}
In the \emph{unanimity} framework, the output class is  $y_i$ only if
\emph{all} of the component models output $y_i$. If there is any disagreement among the models, the input is rejected: $\mathcal{M}(x) = \bot$. For the unanimity framework, joint robustness is achieved when the \emph{unanimity-robust} property defined below is satisfied.

\begin{definition}\label{def:unanimity} Given an input $x$ with true output class $t$ and allowable adversarial distance $\epsilon$, we call a model ensemble \emph{unanimity-robust} for input $x$ if there exists no
adversarial example $x'$ such that $\Delta(x, x') \leq \epsilon$, and 
$M_1(x') = M_2(x') = \ldots = M_n(x') \neq t$.
\end{definition}

\shortsection{Majority}
In the \emph{majority} framework, the output class is  $y_i$ only if at least $\lfloor\frac{n}{2}\rfloor + 1$ models agree on it. If there is
no majority output class, the input is rejected. Joint robustness is achieved when the \emph{majority-robust} property defined below is satisfied:

\begin{definition}\label{def:majority} Given an input $x$ with true output class $t$, allowable adversarial distance $\epsilon$, a model ensemble is \emph{majority-robust} for input $x$ if there exists no adversarial example $x'$ such that $\Delta(x, x') \leq \epsilon$, and there is no class $j$ such that $j \neq t$ and 
$$
\left| \; \{ i \; | \; i \in [n] \wedge M_i(x) = j\; \} \right| \geq
\left \lfloor{n/2}\right \rfloor + 1.$$
\end{definition}

\shortsection{Averaging}
In the \emph{averaging} framework, we take the average of the second last
layer output vectors of each of component models to produce the final output. This second last layer vector is
typically a softmax or logits layer. We use $Z_i(x)$ to denote the second last layer output vector of model $M_i$, and define the average of the second last layer
vectors as:
$$Z(x) = \frac{Z_1(x) + Z_2(x) + \ldots + Z_n(x)}{n} $$
Then, the output of the ensemble is:
$$\mathcal{M}(x) = \operatorname*{argmax}_j Z(x)_j. $$
Joint robustness for an averaging ensemble is satisfied when the \emph{averaging-robust} property defined below is satisfied.

\begin{definition}\label{def:averaging} Given an input $x$ with true output class $t$, allowable adversarial distance $\epsilon$, we call a model ensemble \emph{averaging-robust} if there exists no
adversarial example $x'$ such that $\Delta(x, x') \leq \epsilon$, and there is no class $j$ such that $j \neq t$ and $\mathcal{M}(x') = j$, where $\mathcal{M}(\cdot)$ is as defined above. 
\end{definition}

\section{Certifying ensemble defenses}\label{sec:certifying}

In this section we introduce our techniques to certify a model ensemble is robust for a given input. Our approach extends the single model methods of Wong and Kolter~\cite{wong2018provable} and Tjeng et al.~\cite{tjeng2017} 
to support certification for model ensembles using the different composition mechanisms.

\subsection{Unanimity and majority frameworks}
The simplest approach for certifying joint robustness for the unanimity and majority frameworks would be to certify the robustness of each model in the ensemble individually for a given input, and then make a joint certification decision based on those individual certifications. This strategy is simple but prone to false negatives.

For the unanimity framework, we can verify that an ensemble is unanimity-robust for input $x$ if at least one of the $n$ models is individually robust for $x$. This provides a simple way to use single-model certifiers to verify robustness for an ensemble, but is stricter than what is required to satisfy Definition~\ref{def:unanimity} since compromising a unanimity ensemble requires finding a single input that is a successful adversarial example against every component model.  Hence, this method may substantially  underestimate the actual robustness, especially when the component models have mostly disjoint vulnerability regions. An input that cannot be certified using the this technique, may still be unanimity-robust. 
Nevertheless, this technique is an easy way to establish a lower bound for joint robustness. 

Similarly, for the majority framework, we can use this approach to verify that an ensemble satisfies majority-robustness (Definition~\ref{def:majority}) by checking if at least $\lfloor\frac{n}{2}\rfloor + 1$ models are individually robust for input $x$. As with the unanimity case, this underestimates the actual robustness, but provides a valid joint robustness lower bound. As we will see in Section~\ref{sec:certifyingresults}, the independent evaluation strategy works fairly well for the unanimity framework, but it is almost useless for the majority framework when the number of models in the ensembles gets large.

\subsection{Averaging models}\label{sec:certifyjoint}

As the averaging framework essentially obtains a single model by combining the $n$
models in the ensemble, we can simply apply the single model certification
techniques to that to achieve robust certification. This gives us robust certification according to Definition~\ref{def:averaging}. Furthermore, we can show that this
certification technique \emph{implies} a certification guarantee for the
unanimity framework. In fact, the certification guarantee for the unanimity framework achieved this way has lower false negative rate than the independent technique described in the previous subsection. We state this formally in Theorem~\ref{thm:averaging}, and provide a proof below.

\begin{theorem}\label{thm:averaging}
If for a given input $x$, the averaging ensemble $\mathcal{M}(x)$ is certified to be robust, then the component models $M_1(\cdot),  M_2(\cdot), \ldots, M_n(\cdot)$ combined with the unanimity framework is also certifiably robust according to Definition~\ref{def:unanimity}.
\end{theorem}
 
\begin{proof}
Let $M_1(\cdot),  M_2(\cdot), \ldots, M_n(\cdot)$ be the $n$ component models of the averaging ensemble, and let $Z_1(\cdot), Z_2(\cdot), \ldots, Z_n(\cdot)$ be the second last layer output vectors for each of these models. As described in section 3, given input $x$ we can define the average of the second last layer outputs as:
$$Z(x) = \frac{Z_1(x) + Z_2(x) + \ldots + Z_n(x)}{n}.$$
The final output class $\mathcal{M}(x)$ is defined as:
$$\mathcal{M}(x) = \operatorname*{argmax}_j Z(x)_j.$$
For input $x$, let $t$ be the true output class and $u$ be the target output class.  Now, if averaging ensemble $\mathcal{M}(x)$ is robust, we can write $Z(x)_t > Z(x)_u$. It follows that $Z(x)_t > Z(x)_u$. Thus,
$$
\sum_{i=1}^{n} Z_i(x)_t > \sum_{i=1}^{n} Z_i(x)_u.
$$
This implies that either
$Z_1(x)_t > Z_1(x)_u$ or $\sum_{i=2}^{n} Z_i(x)_t > \sum_{i=2}^{n} Z_i(x)_u$. Generalizing for any $i$, we must have $M_1(\cdot) \;\text{is robust}\; \text{or} \; M_2(\cdot) \; \text{is robust}\; \text{or} \; \ldots M_n(\cdot) \; \text{is robust}$. Thus, in $\mathcal{M}(x)$, an unanimity ensemble of $M_1, M_2, \ldots, M_n$ is unanimity-robust for target class $u$ according to Definition~\ref{def:unanimity}. Thus, if we can show that a model ensemble is \emph{averaging-robust} for all target classes for an input $x$, then it implies that the unanimity ensemble formed with models $M_1, M_2, \ldots M_n$ is also \emph{unanimity-robust} for input $x$.
\end{proof}

This is again a stricter definition of robustness compared to the unanimity-robustness defined in Definition~\ref{def:unanimity}. This means, even though certification of averaging-robustness implies  unanimity-robustness, the opposite is not true. That is, unanimity-robustness does not imply averaging-robustness. Therefore, we again get a lower bound of unanimity-robustness. However, the averaging-robustness is a less strict definition of robustness than the implicit independent certification definition described in the previous subsection. Thus, this formulation gives us a better estimate of true unanimity-robustness.

In this project, we extend two different single-model certification techniques to provide robustness certification for ensembles. The two different techniques we use are described below:

\textbf{Using MIP verification}: Tjeng et al.~\cite{tjeng2017} have used mixed integer programming (MIP) techniques to evaluate robustness of models against adversarial examples. We apply their certification technique on our averaging ensemble model $\mathcal{M}(\cdot)$ to certify the joint robustness of $M_1, M_2, \ldots M_n$.  However, we found this approach to be computationally intensive, and it is hard to scale to larger models.  Nevertheless we found some interesting results for two very simple MNIST models which we report in the next section.

\textbf{Using convex adversarial polytope}: In order to scale our verification technique to larger models, we next extended the dual network formulation by Wong and Kolter~\cite{wong2018provable} to be able to handle the final averaging layer of the averaging ensemble model $\mathcal{M}(\cdot)$. Because this layer is a linear operation, it can be simulated using a fully connected linear layer in the neural network. And because linear networks are already supported by their framework, our averaging model can thus be verified.

\section{Experiments}\label{sec:experiments}

This section reports on our experiments extending two different certification
techniques, MIPVerify (Section~\ref{subsec:mip}) and convex adversarial polytope (Section~\ref{subsec:mip}), for use with model
ensembles in different frameworks. To conduct the experiments, we produced a set of robust models that are trained to be diverse in particular ways (Section~\ref{subsec:trainingdiverse}) and can be combined in various ensembles. Because of the computational challenges in
scaling these techniques to large models, most of our results are only for the convex adversarial polytope method and for now we only have experimental
results on MNIST. Although this is a simple dataset, and may not be
representative of typical tasks, it is sufficient for exploring methods for testing joint vulnerability, and for providing some insights
into the effectiveness of different types of ensembles. 

\subsection{Training Diverse Robust Models}\label{subsec:trainingdiverse}
To train the models in the ensemble frameworks, we used
the cost-sensitive robustness framework by Zhang et al.~\cite{zhang2018}, which
is implemented based on the convex adversarial polytope work. Cost-sensitive robustness provides a principled way to train diverse models.

Cost-sensitive robust training uses a cost-matrix to specify seed-target
class pairs that are trained to be robust. If $C$ is the cost matrix, $i$ is a
seed class, and $t$ is a target class, then $C_{i,j}=1$ is set when we want to
make the trained model robust against adversarial attacks from seed class $i$ to target class $j$, and $C_{i,j}=0$ is set when we don't want to make the model robust for this particular seed-target pair. For the MNIST dataset, $C$ is a $10 \times 10$ matrix. We configure this cost matrix in different ways to produce different types of model ensembles. This provides a controlled way to produce models with diverse robustness properties, in contrast to ad-hoc diverse training methods that vary model architectures or randomize aspects of training. We expect both types of diversity will be useful in practice, but leave exploring ad-hoc diversity methods to future work.

We conduct experiments on ensembles of two, five, and ten models, trained using different cost matrices. The different ensembles we used are listed below:
\begin{itemize}
  \item Two model ensembles where individual models are:
  \begin{enumerate}
     \item Even seed digits robust and odd seed digits robust.
     \item Even target digits robust and odd target digits robust.
     \item Adversarially-clustered seed digits robust.
     \item Adversarially-clustered target digits robust.
  \end{enumerate}
  \item Five model ensembles with individual models that are:
  \begin{enumerate}
     \item Seed digits modulo-5 robust.
     \item Target digits modulo-5 robust.
     \item Adversarially-clustered seed digits robust.
     \item Adversarially-clustered target digits robust.
  \end{enumerate}
  \item Ten model ensembles:
  \begin{enumerate}
     \item Seed digits robust models.
     \item Target digits robust models.
  \end{enumerate}
\end{itemize} 

A representative selection of different models we use are described in Table~\ref{tab:models}.  The overall robust model is a single model trained to be robust on all seed-target pairs (this is the same as standard certifiable robustness training using the convex adversarial polytope). The other models were trained using different cost-matrices. These cost-matrices are shown in Table~\ref{tab:models}. All these models had the same architecture, and they were trained on $L_{\infty}$ distance of 0.1. Each model had 3 linear and 2 convolutional layers.

\begin{table*}[tb]\label{tab:models}
\centering
\begin{tabular}{ccS[table-format=3.1]S[table-format=3.1]} 
\toprule
      & Cost Matrix  & \multicolumn{1}{c}{Overall Certified} & \multicolumn{1}{c}{Cost-Sensitive} \\ 
Model & ($C_{i,j} = 1$) & \multicolumn{1}{c}{Robust Accuracy} & \multicolumn{1}{c}{Robust Accuracy} \\ \midrule
Overall Robust & $\text{ for all }i,j$ & 72.7\% & 72.7\% \\ 
Even-seeds Robust & $i \in \{0,2,4,6,8\}$ & 38.0\% & 77.5\% \\ 
Odd-targets Robust & $j \in \{1,3,5,7,9\}$ & 21.1\% & 86.5\% \\ 
Seeds (2,3,5,6,8) Robust & $i \in \{2,3,5,6,8\}$ & 38.1\% & 74.0\% \\ 
Targets (0,1,4,7,9) Robust & $ j \in \{0,1,4,7,9\}$ & 11.1\% & 89.7\% \\ \midrule
Seed-modulo-5 = 0 Robust & $i \in \{0,5\}$ & 16.7\% & 88.0\% \\ 
Target-modulo-5 = 3 Robust & $j \in \{3,8\}$ & 8.3\% & 94.0\% \\ 
Seeds (3,5) Robust & $i \in \{3,5\}$ & 15.9\% & 81.1\% \\ 
Targets (1,7) Robust & $j \in \{1,7\}$ & 1.4\% & 97.0\% \\ \midrule
Seed-modulo-10 = 3 Robust & $i \in \{3\}$ & 8.5\% & 84.2\% \\ 
Target-modulo-10 = 7 Robust & $j \in \{7\}$ & 0.2\% & 98.4\% \\ \bottomrule
\end{tabular}
 \caption{Models trained using cost-sensitive robustness for use in ensembles. One representative model is shown from each ensemble for the sake of brevity. We show the robust cost-matrix for each model by listing the $i$ and $j$ values where $C_{i,j} = 1$ ($C_{i,j} = 0$ for all others), as well as its overall robust and cost-sensitive accuracy.}
\end{table*}


The adversarial clustering was done to ensure digits that appear visually most
similar to each other are grouped together. This similarity between a pair of
digits was measured in terms of how easily either digit of the pair can be
adversarially targeted to the other digit. These results are consistent with our
intuitions about visual similarity --- for example, MNIST digits 2, 3, 5, 8 are
visually quite similar, and we also found them to be adversarially similar,
hence clustered together.

\subsection{Certifying using MIPVerify}\label{subsec:mip}
We used the MIP verifier on two shallow MNIST networks. One of the networks had two fully-connected layers, and the other had three fully-connected layers. The two-layer network was trained to be robust on even-seeds and the three-layer network  on odd-seeds.  We used adversarial training using PGD attacks to robustly train the models. Even with adversarial training, however, the models were not really robust. Even at $L_{\infty}$ perturbation of $\epsilon=0.02$, which is very low for MNIST dataset, the models only had robust accuracy of 23\% and 28\% respectively. The reason for the lack of robustness is because the networks were very shallow and lacked a convolutional layer. We could not make the models more complex because doing so makes the robust certification too performance-intensive. Still, even with these non-robust models, we can see some interesting results for the ensemble of the two models. We discuss them below.

To understand the robustness possible by constructing an ensemble of the two models, we compute the minimal $L_1$ adversarial perturbation for 100 test seeds
for the two single networks and the ensemble average network built from them. We used $L_1$ distance because the MIP verifier performs better with this, due its
linear nature, compared to $L_2$ or $L_{\infty}$ distances.  More than 90\% of
the seeds were verified within 240 seconds.  Figure~\ref{fig:fig1} shows the
number of seeds that can be proven robust using MIP verification at a given
$L_1$ distance for each model independently, the maximum of the two models, and
the ensemble average model. The verifier was not always able to find
the minimal necessary perturbation for the ensemble network within the time
limit of 240 seconds. In those cases, we reported the maximum adversarial
distance proven to be safe at the time when time limit exceeded -- which
respresents an upper bound of minimal adversarial perturbation. We note from the
figure that number of examples certified by the ensemble average model is higher
than that for either individual model at all minimal $L_1$ distances.

In general, though, we found the MIP Verificaton does not scale well with networks complex enough to be useful in practice. Deeper networks and use of convolutional layers makes the performance of MIP Verify significantly worse. Furthermore, we found that robust networks were harder to verify than non-robust networks with this framework. Because of this, we decided not to use this approach for the remaining experiments which is more practical networks.

\begin{figure}[tb]
\centering
\includegraphics[width=0.7\textwidth]{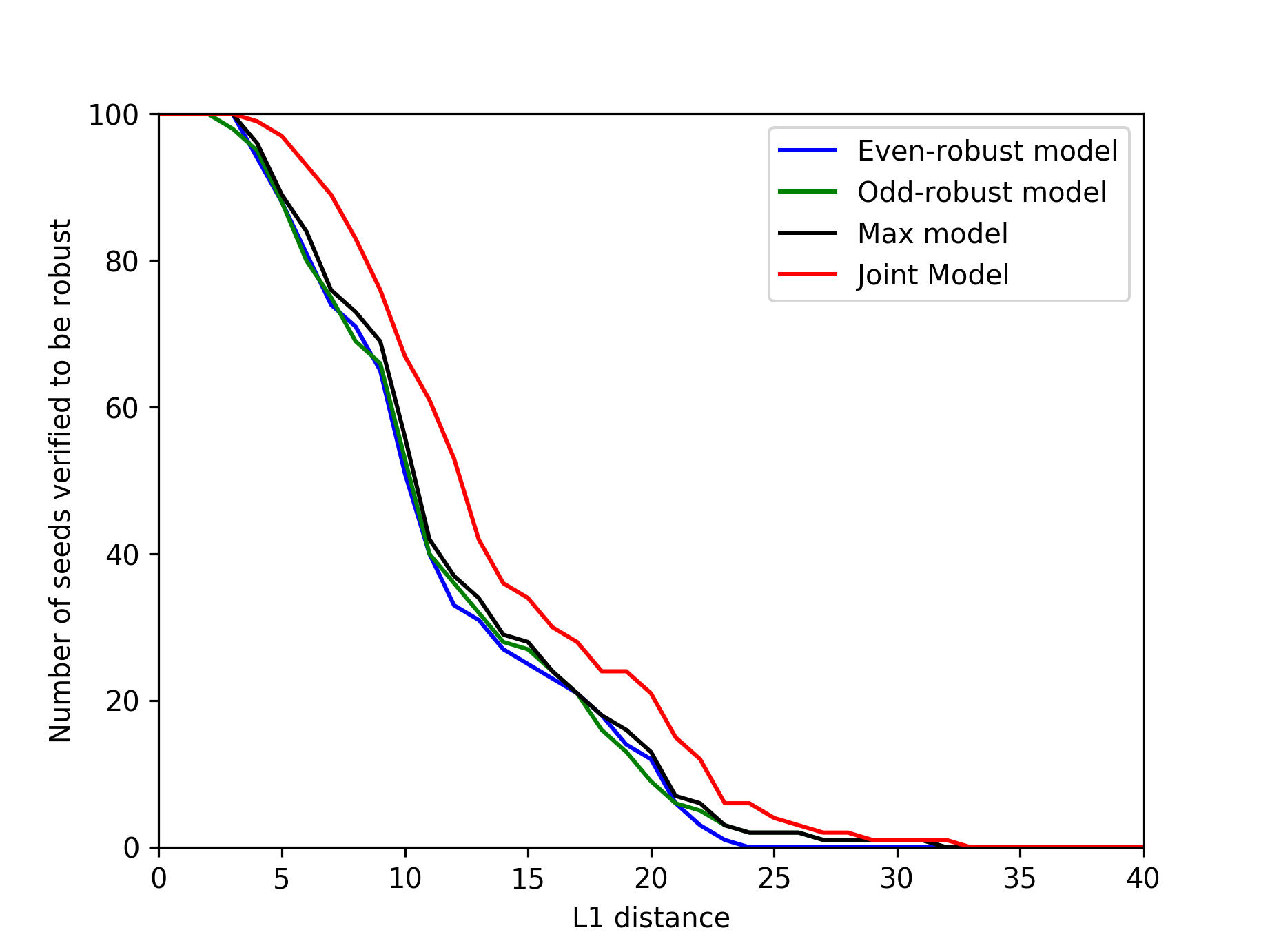}
\caption{Number of test seeds certified to be robust by the single models and
	the ensemble average model for different $L_1$ perturbation distance constraints.}
\label{fig:fig1}
\end{figure}

\subsection{Convex adversarial polytope certification} \label{subsec:convex}\label{sec:certifyingresults}
As the MIP verification does not scale well to larger networks, for our
remaining experiments we use the convex adversarial polytope formulation by Wong
et al.~\cite{wong2018provable}.  We conduct experiments with ensembles of two, five, and ten models, using the models described in Table~\ref{tab:models}. Table~\ref{tab:combined} summarizes the results.

\shortsection{Joint robustness of two-model ensembles}\label{exp:two}
We evaluated two-model ensembles with different choices for the models, using the three composition methods. We ensured that the averaging ensemble could be treated as a single sequential
model made of fully-connected linear layers, so that the robust verification formulation was still valid when applied on it. To do this, we had to first convert the convolutional layers of the single models into linear layers, and then the linear layers of the two models were combined to create larger linear layers for the joint model. We can then calculate the robust error rates of the ensemble average model, as well as the unanimity and majority ensembles for the two-model ensemble. The key here is that no changes were needed to be made to the existing verification framework. 

Table~\ref{tab:combined} shows each ensemble's robust accuracy. For two-model ensembles, the unanimity and the majority frameworks are the same. Thus we can use the same ensemble average technique to certify them. For adversarial clustering into 2-models, we used two clusters -- one for digits (2, 3, 5, 6, 8) and the other for digits (0, 1, 4, 7, 9). 

Compared to the single overall robust model, where 72.7\% of the test examples can be certified robust, with two-model ensembles we can certify up to 78.1\% of seeds as robust (using the averaging composition with the adversarially clustered seed robust models). 

\begin{table*}[tb]
\centering
\begin{tabular}{ccccc} 
\toprule
Models & Composition & Certified Robust & Normal Test Error & Rejection \\\midrule[0.2ex]
Overall Robust & Single & 72.7\% & 5.0\%  & - \\ \midrule[0.2ex]

\multirow{2}{*}{Even/Odd-seed} & Unanimity & 74.7\% & 1.3\%  & 5.0\% \\
 & Average & 75.9\% & 3.3\%  & - \\  \midrule


\multirow{2}{*}{Clustered seed (2)} & Unanimity & 77.3\% & 1.5\%  & 6.0\% \\
 & Average & 78.1\% & 3.0\%  & - \\ \midrule[0.2ex]

\multirow{2}{*}{Seed-modulo-5} & Unanimity & 84.1\% & 0.3\%  & 8.1\% \\
 & Average & 85.3\% & 1.7\%  & - \\ \midrule


\multirow{2}{*}{Clustered seed (5)} & Unanimity & 83.8\% & 0.7\%  & 7.1\% \\
 & Average & 84.3\% & 1.4\%  & - \\ \midrule[0.2ex]


\multirow{2}{*}{Seed-modulo-10} & Unanimity & 85.4\% & 0.1\%  & 9.7\% \\
 & Average & 85.6\% & 1.5\%  & - \\ \bottomrule


\end{tabular}

\caption{Robust certification and normal test error and rejection rates for single model and two, five, and ten-model ensembles for $L_\infty$ adversarial examples with $\epsilon=0.1$.}
\label{tab:combined}
\end{table*}

\begin{figure*}[ht!]
\centering
\subfigure[Cluster Two-Model Ensemble]
{\includegraphics[width=0.65\columnwidth]{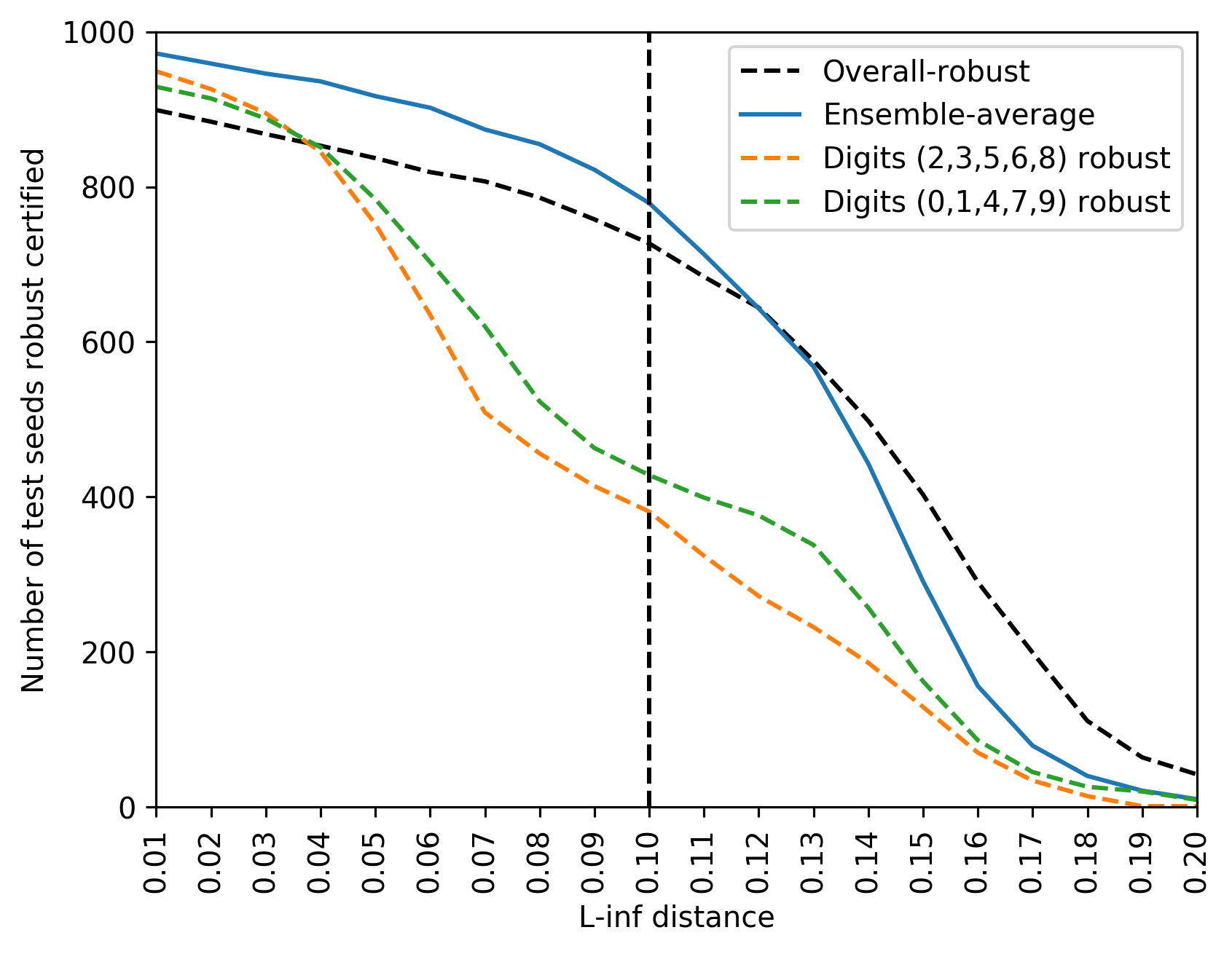}
\label{fig:fig_opt2_seed}
}
\\
\subfigure[Modulo-5 Seeds Ensemble]{
\includegraphics[width=0.47\columnwidth]{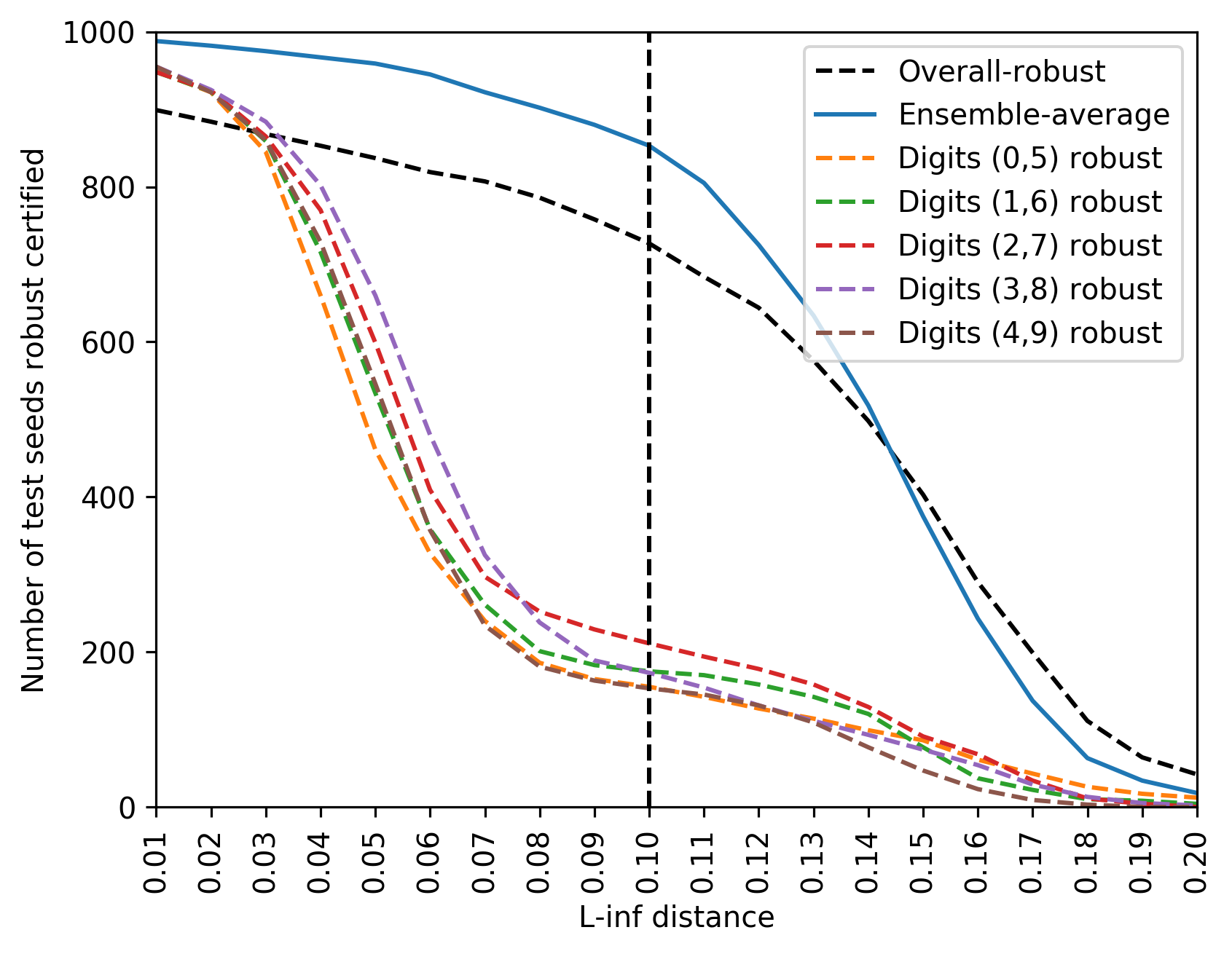}
\label{fig:fig_mod_5_seed}}
\subfigure[Ten-model ensemble]{
\includegraphics[width=0.47\columnwidth]{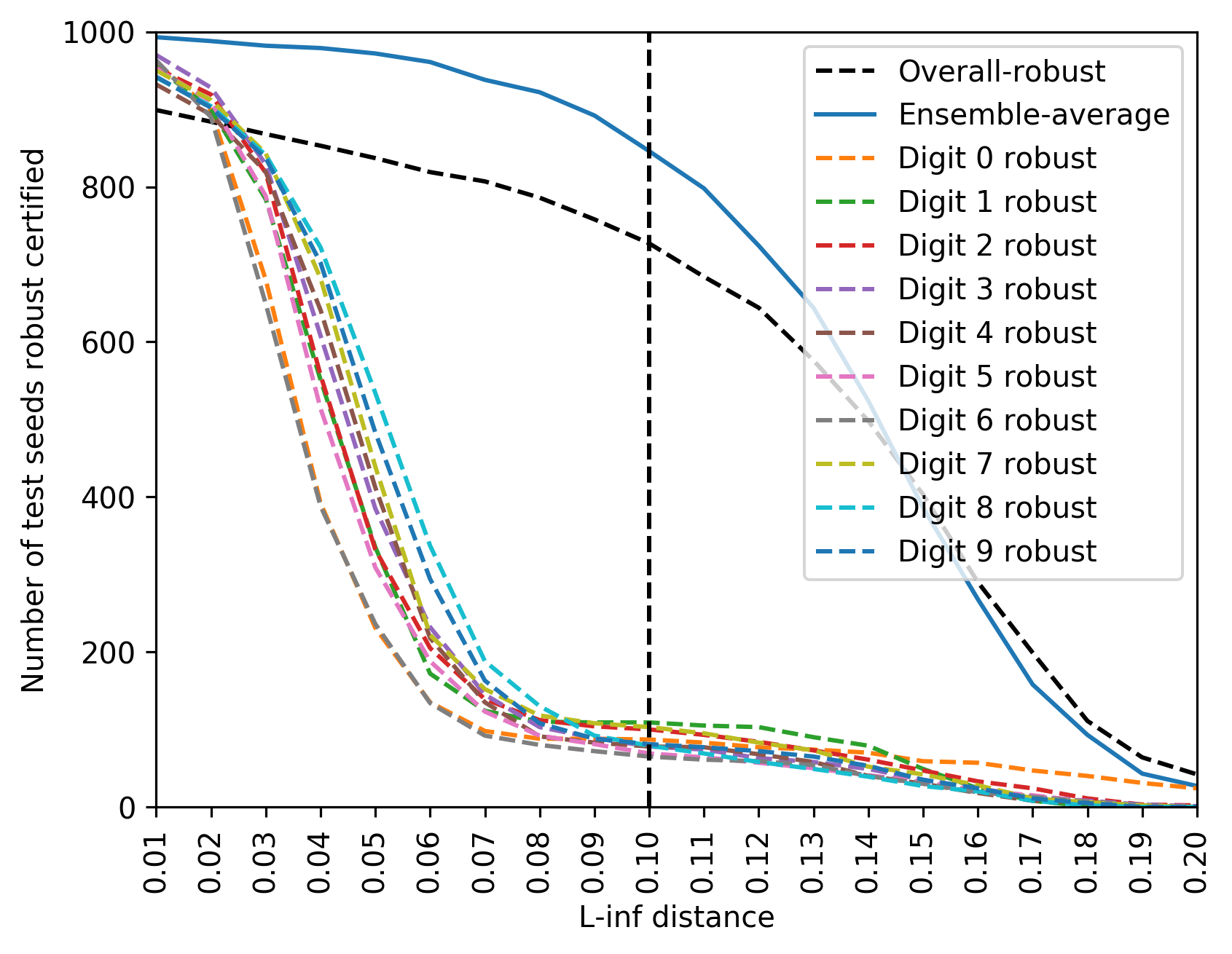}
\label{fig:fig_mod_10_seed}}

\caption{Number of test examples certified to be jointly robust using each model ensemble for different $\epsilon$ values.}\label{fig:varyepsilon}
\end{figure*}

We reran all the above experiments for all $\epsilon$ values from 0.01 to 0.20
to see how the joint robustness changes as the attacks get stronger. Figure~\ref{fig:fig_opt2_seed} shows the results from the adversarially clustered seeds two-model ensemble; the results for the other ensembles show similar patterns and are deferred to Appendix~\ref{sec:appendix}. For $\epsilon$ values up to 0.1, which is the value used for training the robust models, the ensemble model is able to certify more seeds compared to the single overall robust model. We also note that the models that are trained to be target-robust, rather than seed-robust, perform much worse. With even and odd target-robust models we were able to certify only 35.6\% of test examples. We believe the reason for this is that the evaluation criteria of robustness is inherently biased against models that are trained to target-robust. Because, when evaluating, we always start from some test seed, and try to find an adversarial example from that seed -- which is not what the target-robust models are explicitly trained to prevent.

\shortsection{Five-model Ensembles}\label{exp:five}
Our joint certification framework can be extended to ensembles of any number of models. We trained the five models to be robust on modulo-5 seed digits. Ensembles of these models had better certified robustness than the best two-model ensembles. For example, with averaging composition 85.3\% of test examples can be certified robust (compared to our previous best result of 78.1\% with two
models). Figure~\ref{fig:fig_mod_5_seed} shows how the number of certifiable test seeds drops with increasing $\epsilon$, but worth noting is the large gap between any individual model's certifiable robustness and that for the average ensemble. We also trained model by adversarially clustering into 5-models --  for digits (4, 9), (3, 5), (2,8), (0, 6) and (1, 7). For the clustered seed robust ensemble, the results were slightly worse (84.3\%) than modulo-5 seeds robust model. One difference between the two-model and five-model ensembles is that in the latter, the unanimity and the majority frameworks are different. We found that independent certification does not really work for majority framework. We were able to certify almost no test seeds for the majority framework for five-model ensembles.

\shortsection{Ten-model Ensembles}
\label{exp:ten}
Finally, we tried ensembles of ten models, each trained to be robust for a selected seed digit ($0, 1, 2, \ldots, 9$). The certified robust rate of the 10-model ensemble trained to be seed robust was 85.6\%. This is slightly higher
than the 5-model ensemble (85.3\%), but perhaps not worth the extra performance cost. It is notable, though, that the unanimity model reduces the normal test error for this ensemble to 0.1\%. This means that out of 1000 test seeds, 853 were certified to be robust, 48 were correctly classified but could not be certified, 97 were rejected due to disagreement among the models, and 1 was incorrectly classified by all 10 models. Figure~\ref{fig:badexamples} shows the one test example where all models agree on a predicted class but it is not the given label (Figure~\ref{fig:incorrect}, and selected typical rejected examples from the 97 tests where the models disagree (Figure~\ref{fig:rejected}).

\begin{figure*}
\centering

\subfigure[The one test example that is incorrectly classified by all ten models. (Labeled as {\sf 6}, predicted as {\sf 1}.)]{
$\qquad\qquad\qquad\quad$ 
\includegraphics[width=0.1\textwidth]{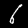}
$\qquad\qquad\qquad\quad$
\label{fig:incorrect}
}
\ \ 
\subfigure[Examples of rejected test examples for which models disagree on the predicted class.]{
\centering
\includegraphics[width=0.1\textwidth]{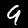} \
\includegraphics[width=0.1\textwidth]{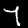} \
\includegraphics[width=0.1\textwidth]{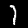} \
\includegraphics[width=0.1\textwidth]{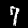}
\label{fig:rejected}}
\caption{Test examples that are misclassified or rejected by the ten-model ensemble.}\label{fig:badexamples}
\end{figure*}

\shortsection{Summary}
Figure~\ref{fig:comparison} compares the robust certification rate for the two, five, and ten-model ensembles. Clustered seed robust models
generally tend to perform well, although just random modulo seed robust models perform almost just as well.  

One potential issue with any ensemble models is the possibility of false positives. In our case, the use of multiple models in the unanimity and majority frameworks also introduce the possibility of \emph{rejecting}
benign inputs. As the number of models in a unanimity  ensemble increases, the rejection rate on normal inputs increases since if any one model disagrees the input is rejected. However, if the false rejection rate is reasonably low, then in many situations that may be an acceptable trade-off for higher adversarial robustness. The results in Table~\ref{tab:combined} are consistent with this, but show that even the ten-model unanimity ensemble has a rejection rate below 10\%. For more challenging classification tasks, strict unanimity composition may not be an option if rejection rates become unacceptable, but could be replaced by relaxed notions (for example, considering a set of related classes as equivalent for agreement purposes, or allowing some small fraction of models to disagree).

\begin{figure}[tb]
\centering
\includegraphics[width=0.7\textwidth]{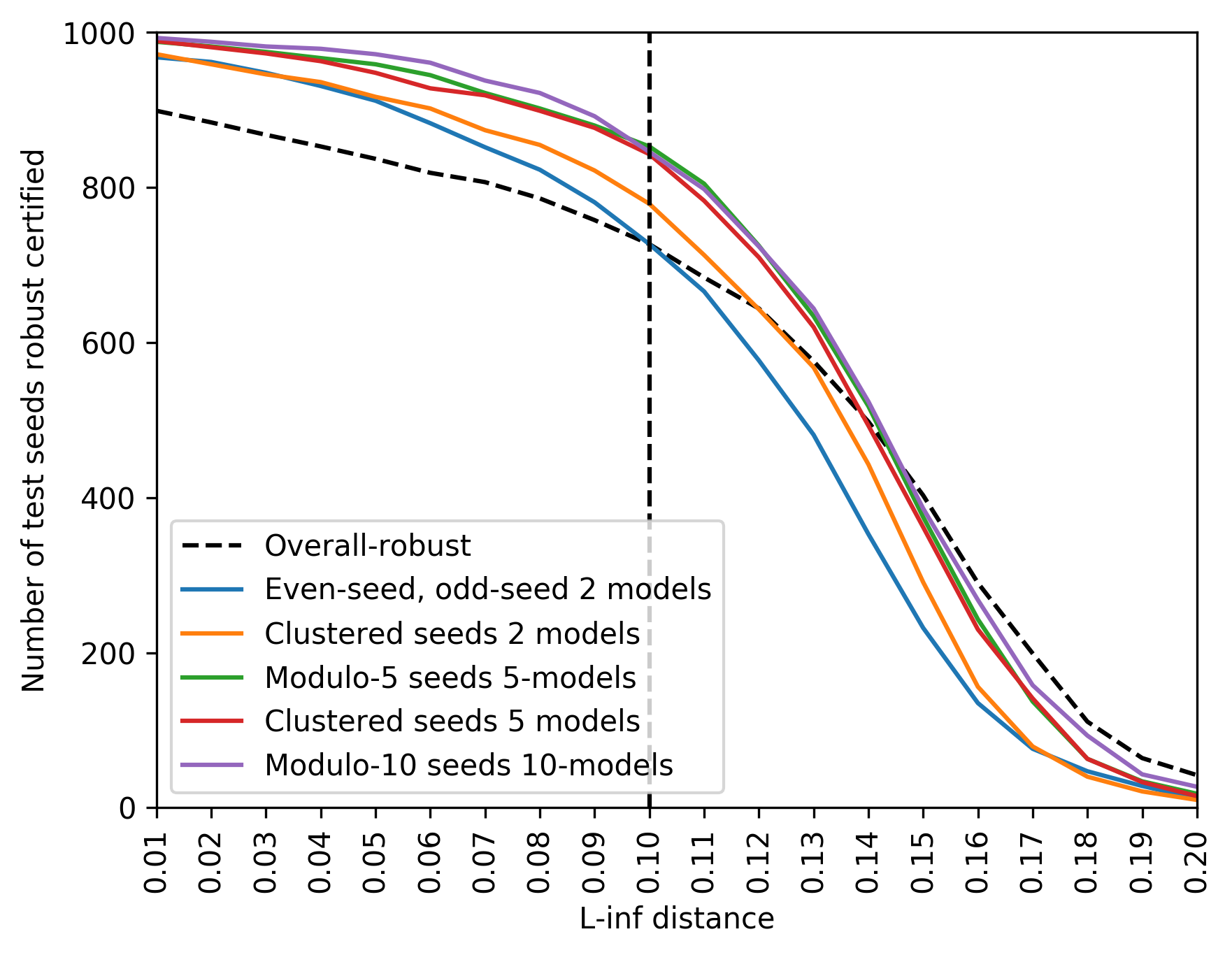}
\caption{Number of test seeds certified to be jointly robust using ten models,
	five models, two models and single model for different $\epsilon$
	values.}
\label{fig:comparison}
\end{figure}

\section{Conclusion} 

We extended robust certification models designed for single models to provide joint robustness guarantees for ensembles of models. Our novel joint-model formulation technique can be used to extend certification frameworks to provide certifiable robustness guarantees that are substantially stronger than what can be obtained using the verification techniques independently. Furthermore, we have shown that cost-sensitive robustness training with diverse cost matrices can produce models that are diverse with respect to joint robustness goals. The results from our experiments suggest that
ensembles of models can be useful for increasing the robustness of models
against adversarial examples. These is a vast space of possible ways to
train models to be diverse, and ways to use multiple models in an ensemble, that may lead to even more robustness. As we noted in our motivation, however, without efforts to certify joint robustness, or to ensure that models in an ensemble are diverse in their vulnerability regions, the apparent effectiveness of an ensemble may be misleading. Although the methods we have used cannot yet scale beyond tiny models, our results provide encouragement that ensembles can be constructed that provide strong robustness against even the most sophisticated adversaries.


\subsection*{Availability}
\noindent
Open source code for our implementation and for reproducing our experiments is available at: \url{https://github.com/jonas-maj/ensemble-adversarial-robustness}. 

\subsection*{Acknowledgements}
\noindent

We thank members of the Security Research Group, Mohammad Mahmoody,
Vicente Ord\'{o}\~{n}ez Rom\'{a}n, and Yuan Tian for helpful comments on this work, and thank Xiao Zhang, Eric Wong, and  Vincent Tjeng, Kai Xiao, and Russ Tedrake for their open source projects that we made use of in our experiments. This research was sponsored in part by the National Science Foundation \#1804603 (Center for Trustworthy Machine Learning, SaTC Frontier: End-to-End Trustworthiness of Machine-Learning Systems), and additional support from Amazon, Google, and Intel.

\clearpage
\bibliographystyle{plain}
\bibliography{references, aml}
\newpage
\appendix

\section{Additional Experimental Results}
\label{sec:appendix}

\begin{figure}[h]
\centering
\subfigure[Even/odd seeds robust]{
\includegraphics[width=0.3\textwidth]{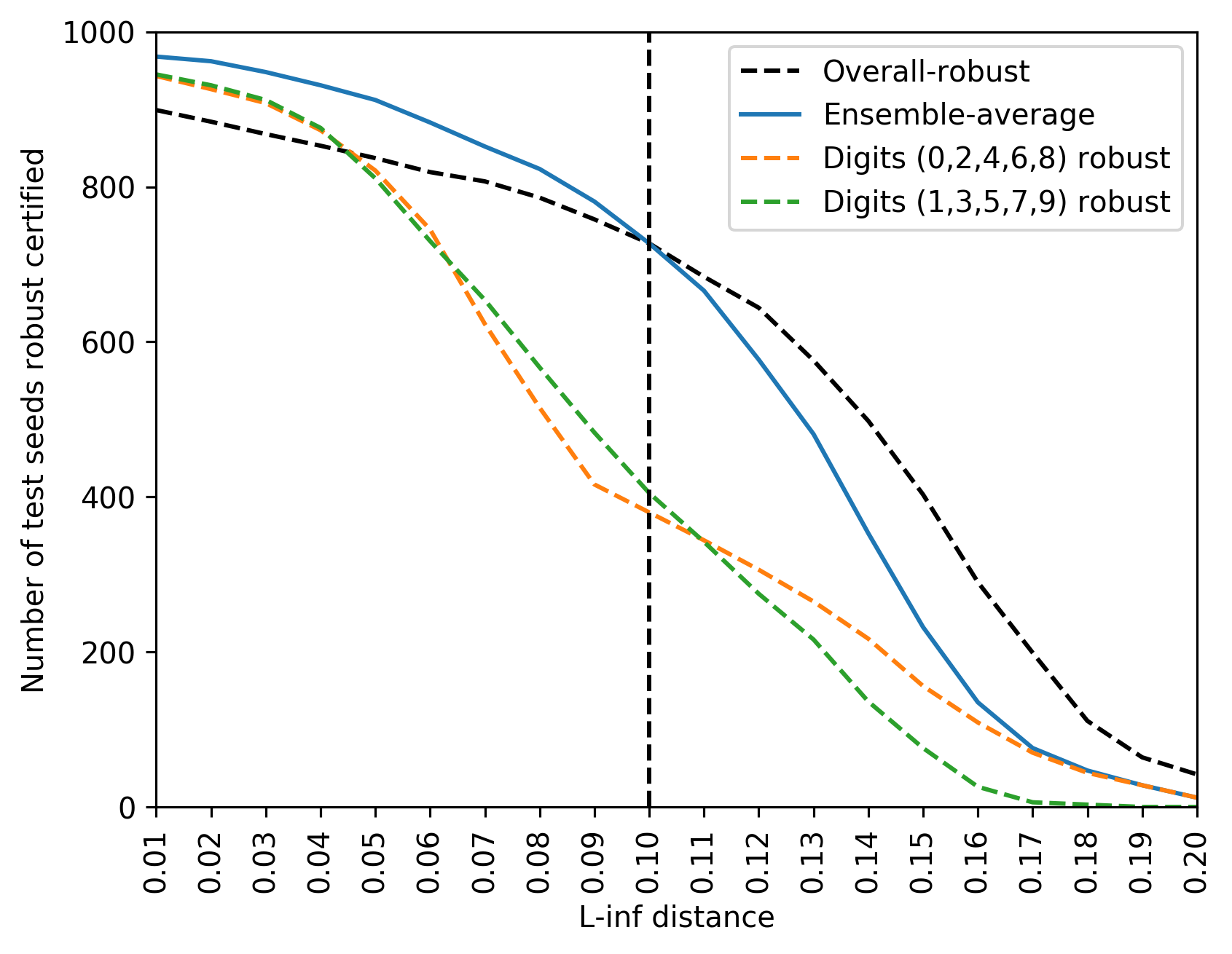}
\label{fig:fig_mod_2_seed}}
\ 
\subfigure[Even/odd targets robust]{
\includegraphics[width=0.3\textwidth]{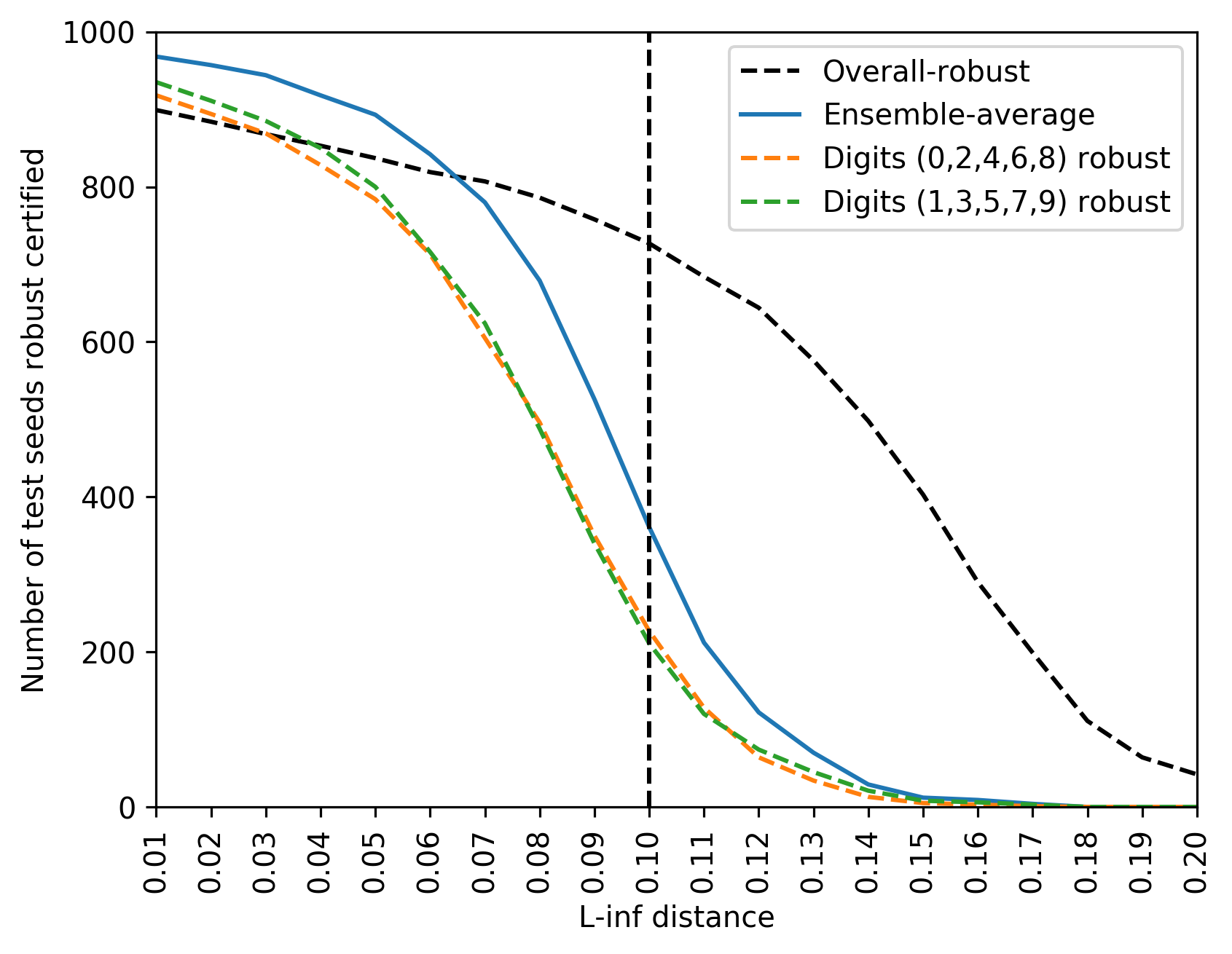}
\label{fig:fig_mod_2_target}}
\ 
\subfigure[Clustered Targets]{
\includegraphics[width=0.3\textwidth]{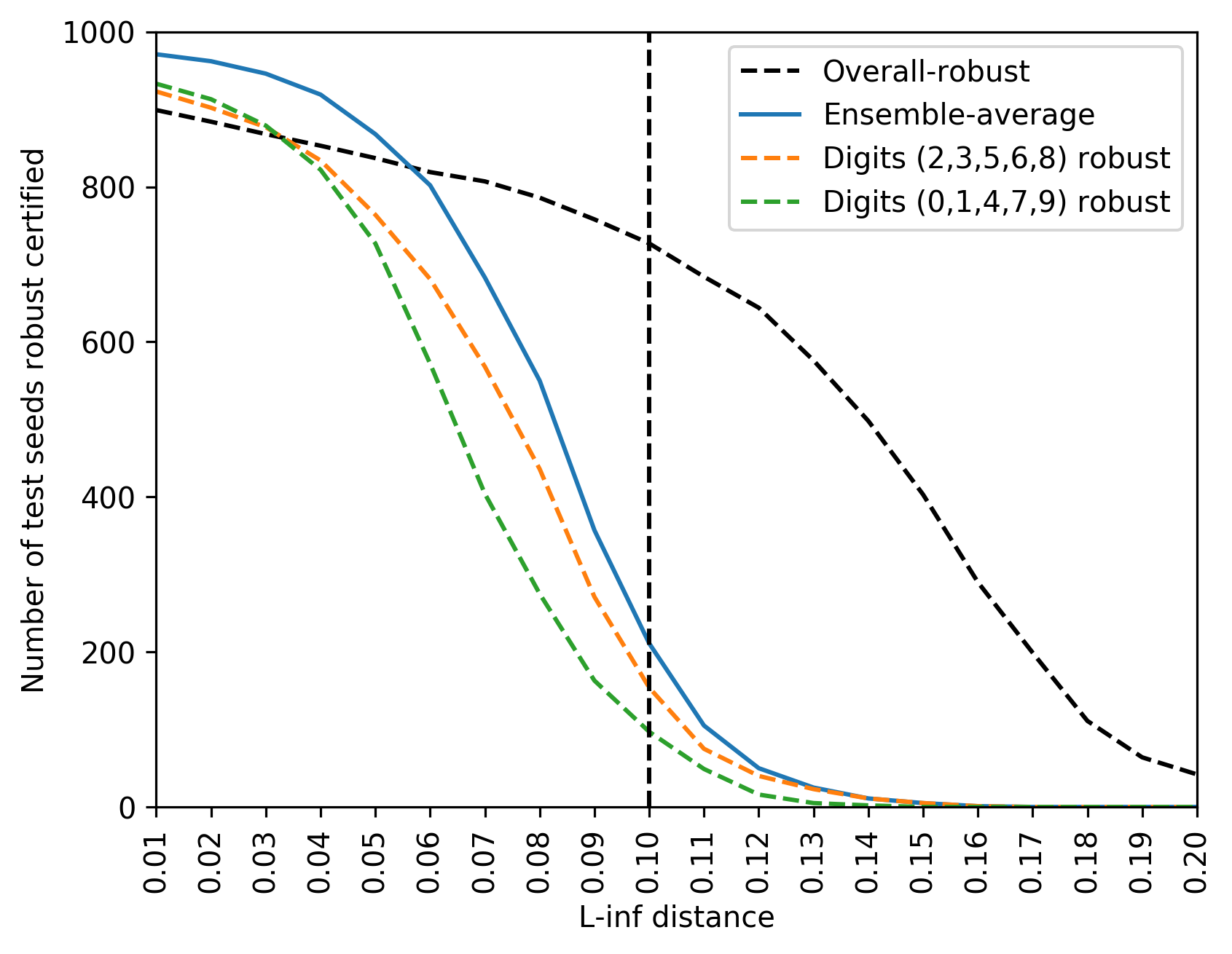}
\label{fig:fig_opt_2_target}
}
\caption{Number of test seeds certified to be jointly robust using the
	individual models and different two-model ensembles average framework for different $\epsilon$ values.}
\end{figure}

\begin{figure}[ht!]
\centering
\subfigure[Modulo-5 Targets Robust]{
\includegraphics[width=0.3\textwidth]{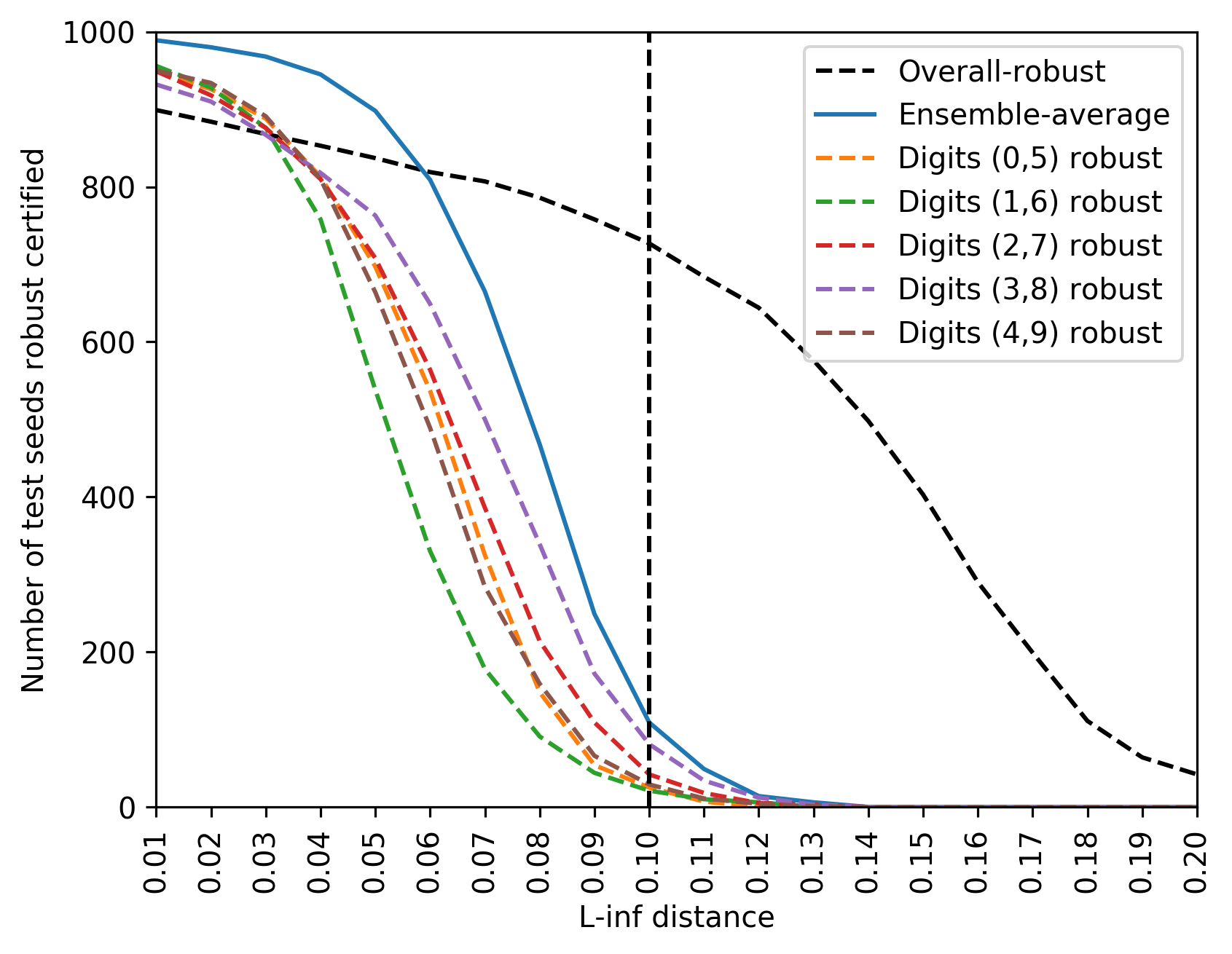}
\label{fig:fig_mod_5_target}}
\ 
\subfigure[Modulo-5 Seeds Robust]{
\includegraphics[width=0.3\textwidth]{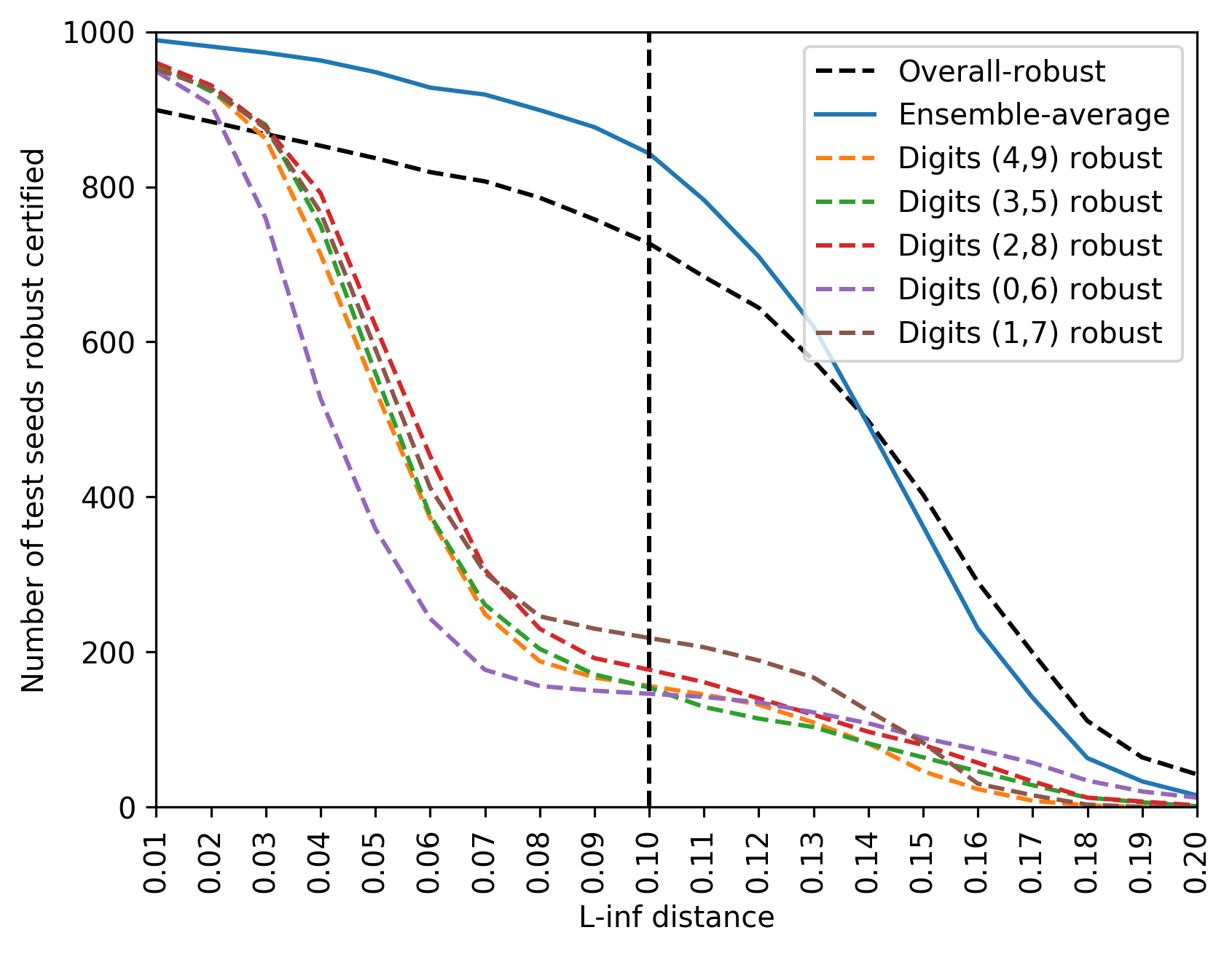}
\label{fig:fig_opt_5_seed}
}
\ 
\subfigure[Clustered Targets]{
\includegraphics[width=0.3\textwidth]{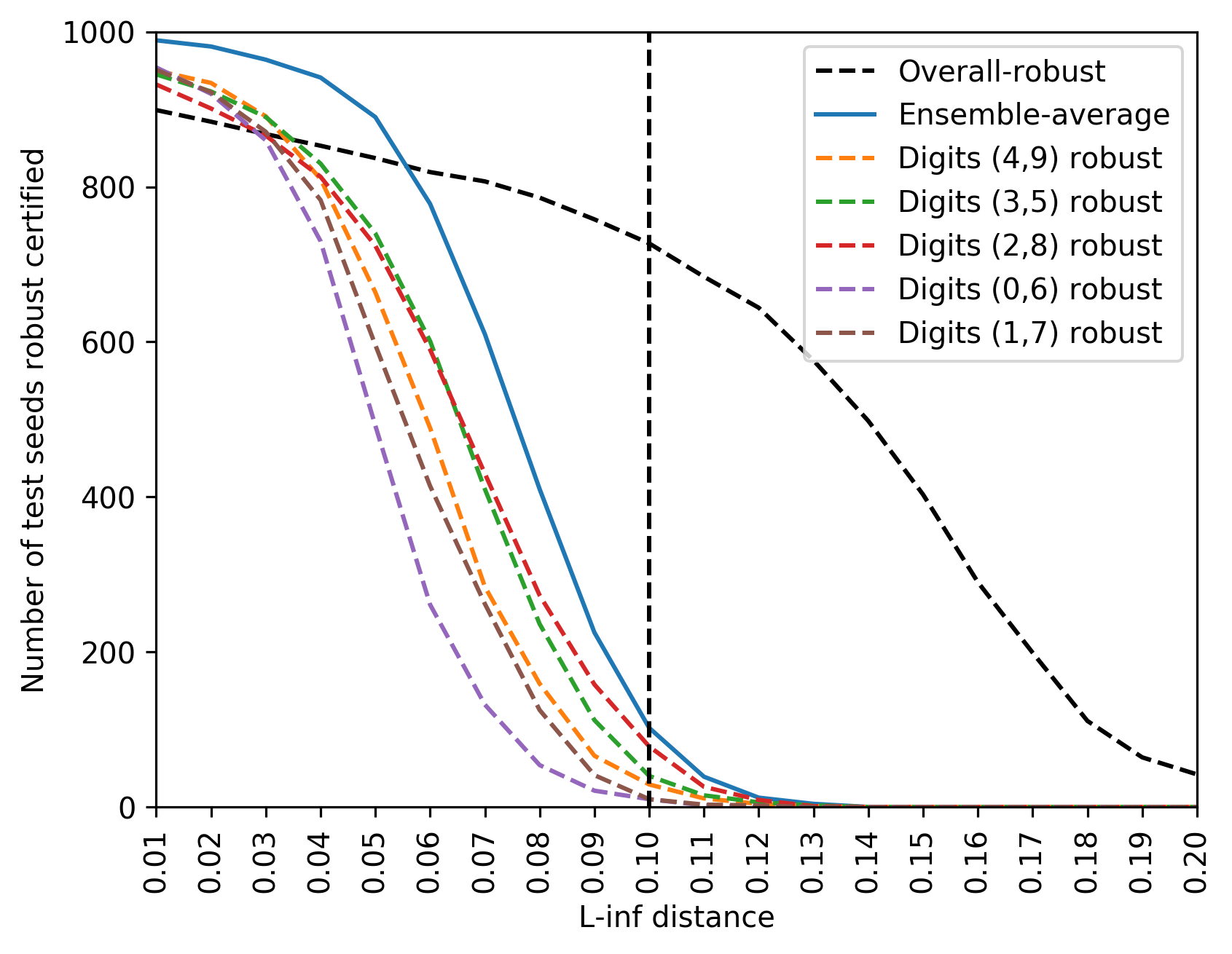}
\label{fig:fig_opt_5_target}
}
\caption{Number of test examples certified to be jointly robust using the
	individual models and the 5-model ensembles average framework with different $\epsilon$ values.}
\end{figure}

\begin{figure}[ht!]
\centering
\includegraphics[width=0.5\textwidth]{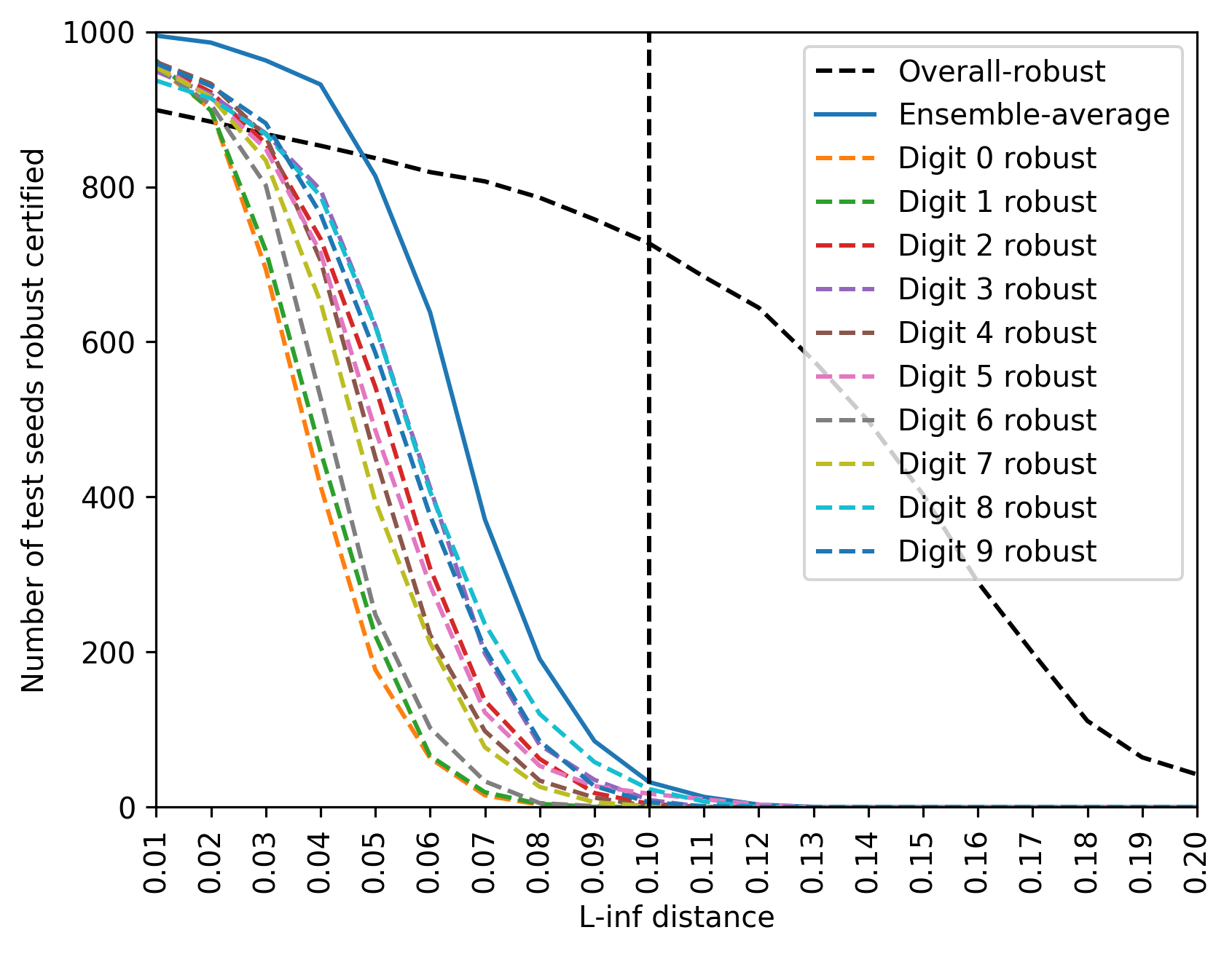}
\caption{Number of test examples certified to be jointly robust using the
	individual models and the 10-model ensemble average framework with
	targets robust for different $\epsilon$ values.}
\label{fig:fig_mod_10_target}
\end{figure}

\end{document}